%% file: main_CPAIOR.tex
\documentclass[runningheads]{llncs}
\usepackage[T1]{fontenc}
\usepackage{graphicx}
\input{math_commands.tex}

\usepackage{amssymb}
\usepackage[normalem]{ulem}
\usepackage{soul}
\usepackage{enumerate}
\usepackage{pgfplots}
\usepackage[symbol]{footmisc}

\definecolor{bblue}{rgb}{0,0.15,0.25}
\definecolor{bmblue}{rgb}{0,0.51,0.73}
\definecolor{borange}{rgb}{0.97,0.3,0.09}
\definecolor{cred}{rgb}{0.39,0,0}

\usepackage{xparse}
\ExplSyntaxOn
\DeclareExpandableDocumentCommand \round { O{0} m }
 { \fp_eval:n { round(#2,#1) } }
\ExplSyntaxOff

\newcommand{\underlineAccuracyLabel}[0]{\setulcolor{bmblue}\ul{Test accu}\setulcolor{borange}\ul{racy (\%)} \setulcolor{black}}

\newcommand{\accuracyPlot}[9]{
    \begin{tikzpicture}[scale=#1]
        \begin{axis}[
            cycle list name=mark list,
            table/col sep=comma,
            legend image post style={scale=0.3},
            legend style={at={(1,1)},anchor=north east},
            xtick=data,
            xticklabels from table={data/#2.csv}{global_sparsity},
            xlabel={\textbf{\Large #3}},
            x label style={at={(axis description cs:0.5,1.25)},anchor=north},
            ylabel={\huge #5},
            xticklabel style={font=\large, rotate = 70},
            yticklabel style={font=\large},
            ytick pos=left,
            ]
            \addplot+[bmblue, 
                line width={3pt},
                scatter,
                mark=none]
            table[x index=0,y expr=100*\thisrow{ACC_U}]
            {data/#2.csv};
            \addlegendentry{{Baseline}}
            \addplot+[borange, 
                line width={3pt},
                scatter,
                mark=none]
            table[x index=0,y expr=100*\thisrow{ACC_O}]
            {data/#2.csv};
            \addlegendentry{Ours}
        \end{axis}
        \begin{axis}[
        ylabel near ticks, yticklabel pos=right,
        table/col sep=comma, 
        nodes near coords,
        enlarge y limits={#4},
        xticklabels={,,},
        yticklabel style={font=\large},
        ylabel={\huge \color{cred} #6},
        ylabel style={rotate=180},
        xlabel={\huge #7},
        x label style={at={(axis description cs:0.5,#8)},anchor=south},
        ytick pos=right,
        ]
        \addplot[cred, ybar] table[x index=0, y expr={\round[1]{100*(\thisrow{ACC_O}-\thisrow{ACC_U})}}] {data/#2.csv};
        \end{axis}
    \end{tikzpicture}\hspace*{#9}
} %

\newcommand{\uboundPlot}[6]{
    \begin{tikzpicture}[scale=#1]
        \begin{axis}[
            cycle list name=mark list,
            table/col sep=comma,
            xtick=data,
            xticklabels from table={data/#2.csv}{global_sparsity},
            xlabel={\textbf{\Large #3}},
            x label style={at={(axis description cs:0.5,1.25)},anchor=north},
            ylabel={\huge \color{bmblue} \dashuline{#4}},
            xticklabel style={font=\large, rotate = 70},
            yticklabel style={font=\large},
            ytick pos=left,
            ]
            \addplot+[bmblue, 
                dashed, 
                line width={3pt},
                scatter,
                mark=none]
            table[x index=0,y=UB_U]
            {data/#2.csv};
        \end{axis}
    \begin{axis}[
            ylabel near ticks, yticklabel pos=right,
            cycle list name=mark list,
            table/col sep=comma,
            xtick=data,
            xticklabels={},
            xlabel={\huge Network density $p$},
            x label style={at={(axis description cs:0.5,-0.25)},anchor=south},
            ylabel={\huge \color{bmblue} \uline{#5}},
            ylabel style={rotate=180},
            xticklabel style={font=\small, rotate = 70},
            yticklabel style={font=\large},
            ytick pos=right,
            legend style={at={(0.2,0.6)}},
            ]
            \addplot+[bmblue, 
                line width={3pt},
                scatter,
                mark=none]
            table[x index=0,y expr=100*\thisrow{ACC_U}]
            {data/#2.csv};
        \end{axis}
    \end{tikzpicture}\hspace*{#6}
}

\newcommand{\acclrPlot}[6]{
    \begin{tikzpicture}[scale=#1]
        \begin{axis}[
            cycle list name=mark list,
            table/col sep=comma,
            legend columns=2,
            legend style={at={(1,1)},anchor=north east},
            xtick=data,
            xticklabels from table={data/#2.csv}{global_sparsity},
            xlabel={\textbf{\Large #3}},
            x label style={at={(axis description cs:0.5,1.25)},anchor=north},
            ylabel={\huge \color{olive} \uline{#4}},
            xticklabel style={font=\large, rotate = 70},
            yticklabel style={font=\large},
            ytick pos=left,
            ]
            \addplot+[olive, 
                line width={3pt},
                scatter,
                mark=none]
            table[x index=0,y=LR_U]
            {data/#2.csv};
        \end{axis}
    \begin{axis}[
            ylabel near ticks, yticklabel pos=right,
            cycle list name=mark list,
            table/col sep=comma,
            xtick=data,
            xticklabels={},
            xlabel={\huge Network density $p$},
            x label style={at={(axis description cs:0.5,-0.25)},anchor=south},
            ylabel={\huge \color{bmblue} \uline{#5}},
            ylabel style={rotate=180},
            yticklabel style={font=\large},
            ytick pos=right,
            ]
            \addplot+[bmblue, 
                line width={3pt},
                scatter,
                mark=none]
            table[x index=0,y expr=100*\thisrow{ACC_U}]
            {data/#2.csv};
        \end{axis}
    \end{tikzpicture}\hspace*{#6}
}

\begin{document}
\title{Getting Away with More Network Pruning: \\ 
From Sparsity to Geometry and Linear Regions}
\titlerunning{Getting Away with More Network Pruning}
%
\author{Junyang Cai\inst{*}\inst{1} \and
Khai-Nguyen Nguyen\inst{*}\inst{1} \and
Nishant Shrestha\inst{1} \and
Aidan Good\inst{1} \and \\
Ruisen Tu\inst{1} \and
Xin Yu\inst{2} \and
Shandian Zhe\inst{2} \and
Thiago Serra\inst{1}
}
\authorrunning{J. Cai, K-.N. Nguyen, et al.}
%
\institute{Bucknell University, United States \\
\email{\{jc092,nkn002,ns037,wag011,rt024,thiago.serra\}@bucknell.edu} \and
University of Utah, United States \\
\email{xin.yu@utah.edu,zhe@cs.utah.edu}}
\maketitle              

\renewcommand{\thefootnote}{\fnsymbol{footnote}}
\footnotetext[1]{Equal contribution.}

\begin{abstract}
One surprising trait of neural networks is the extent to which their connections can be pruned with little to no effect on accuracy.
But when we cross a critical level of parameter sparsity, 
pruning any further leads to a sudden drop in accuracy. 
This drop plausibly reflects a loss in model complexity, 
which we aim to avoid. 
In this work, 
we explore how sparsity also affects the geometry of the linear regions defined by a neural network, 
and consequently reduces the expected maximum number of linear regions based on the architecture. 
We observe that pruning affects accuracy similarly to how sparsity affects the number of linear regions and our proposed bound for the maximum number. 
Conversely, 
we find out that selecting the sparsity across layers to maximize our bound very often improves accuracy in comparison to pruning as much with the same sparsity in all layers,  
thereby providing us guidance on where to prune.  

\end{abstract}
\section{Introduction}
In deep learning, 
there are often good results with little justification and good justifications with few results. 
Network pruning exemplifies the former: we can easily prune half or more of the connections of a neural network without affecting the resulting accuracy, but we may have difficulty explaining why we can do that. 
The theory of linear regions exemplifies the latter: we can theoretically design neural networks to express very nuanced functions, but we may end up obtaining much simpler ones in practice. 
In this paper, we posit that the mysteries of pruning and the wonders of linear regions can complement one another.

When it comes to pruning, 
we can reasonably argue that reducing the number of parameters improves generalization. 
While Denil et al.~\cite{denil2013parameters} show that the parameters of neural networks can be redundant, 
it is also known that the smoother loss landscape of larger neural networks leads to better training convergence~\cite{li2018landscape,ruoyu2020landscape}. 
Curiously, Jin et al.~\cite{jin2022generalization} argue that pruning also smooths the loss function, 
which consequently improves convergence during fine tuning --- the additional training performed after pruning the network. 
However, it remains unclear to what extent we can prune without ultimately affecting accuracy, 
which is an important concern since a machine learning model with fewer parameters can be deployed more easily in environments with limited hardware.  

The survey by Hoefler et al.~\cite{hoefler2021sparsity} illustrates that a moderate amount of pruning typically improves accuracy while further pruning may lead to a substantial decrease in accuracy, 
whereas Liebenwein et al.~\cite{liebenwein2021lost} show that this tolerable amount of pruning depends on the task for which the network is trained. 
In terms of what to prune, another survey by Blalock et al.~\cite{blalock2020survey} observes that most approaches consist of  either
removing parameters with the smallest absolute value  \cite{hanson1988minimal,mozer1989relevance,janowsky1989prunning,han2015connections,han2016deepcompression,li2017cnn,frankle2019lottery,elesedy2020lth,gordon2020bert,tanaka2020synapticflow,liu2021sparse}; or 
removing parameters with smallest expected impact on the output   \cite{lecun1989damage,hassibi1992surgeon,hassibi1993surgeon,lebedev2016braindamage,molchanov2017taylor,dong2017layerwise,yu2018importance,zeng2018mlprune,baykal2019coresets,lee2019pretraining,wang2019eigendamage,liebenwein2020provable,wang2020tickets,xing2020lossless,singh2020woodfisher,yu2022cbs}, to which we can add the special case of exact compression  \cite{serra2020lossless,sourek2021lossless,serra2021compression,ganev2022compression}.

While most work on this topic has helped us prune more with a lesser impact on accuracy,  
fairness studies recently debuted by Hooker et al.~\cite{hooker2019forget} have focused instead on the impact of pruning on recall --- the ability of a network to correctly identify samples as belonging to a certain class. 
Recall tends to be more severely affected by pruning in classes and features that are underrepresented in the  dataset~\cite{hooker2019forget,paganini2020responsibly,hooker2020bias}, 
which Tran et al.~\cite{tran2022disparate} attribute to differences across such groups in gradient norms and Hessian matrices of the loss function. 
In turn, Good et al.~\cite{good2022recall} showed that such recall distortions may also occur in balanced datasets, but in a more nuanced form: 
moderate pruning leads to comparable or better accuracy while reducing differences in recall, 
whereas excessive pruning leads to lower accuracy while increasing differences in recall. 
Hence, 
avoiding a significant loss in accuracy due to pruning is also relevant for fairness.

Overall, 
network pruning studies have been mainly driven by one question: \textbf{how can we get away with more network pruning?}
Before we get there with our approach, let us consider the other side of the coin in our narrative.

When it comes to the theory of linear regions, 
we can reasonably argue that the number of linear regions may represent the expressiveness of a neural network --- and therefore relate to its ability to classify more complex data. 
We have learned that a neural network can be a factored representation of functions that are substantially more complex than the activation function of each neuron. 
This theory is applicable to networks in which the neurons have piecewise linear activations, 
and consequently the networks represent a piecewise linear function in which the number of pieces --- or linear regions --- may grow polynomially on the width and exponentially on the depth of the network~\cite{pascanu2013on,montufar2014on}. 
When the activation function is the Rectified Linear Unit~(ReLU)~\cite{nair2010rectified,glorot2011rectifier}, 
each linear region corresponds to a different configuration of active and inactive neurons. 
For geometric reasons that we discuss later, not every such configuration is feasible.

The study of linear regions bears some resemblance to universal approximation results, 
which have shown that most functions can be approximated to arbitrary precision with sufficiently wide neural networks~\cite{cybenko1989approximation,funahashi1989approximate,hornik1989approximators}. 
These results were extended in~\cite{yarotsky2017relu} to the currently more popular ReLU activation  
and later focused on networks with limited width but arbitrarily large depth~\cite{lu2017expressive,hanin2017approximating}. 
In comparison to universal approximation, 
the theory of linear regions tells us what piecewise linear functions are possible to represent --- and thus what other functions can be approximated with them --- 
in a context of limited resources translated as both the number of layers and the width of each layer. 

Most of the literature is focused on fully-connected feedforward networks using the ReLU activation function, 
which will be our focus on this paper as well. 
Nevertheless, there are also adaptations and extensions of such results for convolutional networks by \cite{xiong2020cnn} and for maxout networks~\cite{goodfellow2013maxout} by \cite{montufar2014on,serra2018bounding,tseran2021expected,montufar2021maxout}. 

Several papers have shown that the right choice of parameters may lead to an astronomical number of linear regions~\cite{montufar2014on,telgarsky2015,arora2018understanding,serra2018bounding}, 
while other papers have shown that 
the maximum number of linear regions can be affected by narrow layers~\cite{montufar2017notes}, 
the number of active neurons across different linear regions~\cite{serra2018bounding}, 
and the parameters of the network~\cite{serra2020empirical}. 
Despite the exponential growth in depth, Serra et al.~\cite{serra2018bounding} observe that a shallow network may in some cases yield more linear regions among architectures with the same number of neurons. 
Whereas the number of linear regions among networks of similar architecture relates to the accuracy of the  networks~\cite{serra2018bounding}, 
Hanin and Rolnick~\cite{hanin2019complexity,hanin2019deep} show that the typical initialization and subsequent training of neural networks is unlikely to yield the expressive number of linear regions that have been reported elsewhere. 

These contrasting results lead to another question: \textbf{is the network complexity in terms of linear regions relevant to accuracy if trained models are typically much less expressive in practice?}
Now that you have read both sides of our narrative, 
you may have guessed where we are heading.

We posit that these two topics --- network pruning and the theory of linear regions --- can be combined. Namely, that the latter can guide us on how to prune neural networks, since it can be a proxy to model complexity. 

But we must first address the paradox in our second question. As observed by Hanin and Rolnick~\cite{hanin2019complexity}, 
perturbing the parameters of networks designed to maximize the number of linear regions, such as the one by Telgarsky~\cite{telgarsky2015}, leads to a sudden drop on the number of linear regions. 
Our interpretation is that every architecture has a probability distribution for the number of linear regions. 
If by perturbing these especially designed constructions we obtain networks with much smaller numbers, we may infer that these constructions correspond to the tail of that distribution. 
However, if certain architectural choices lead to much larger numbers of linear regions at best, 
we may also conjecture that the entire distribution shifts accordingly, 
and thus that even the ordinary trained network might be more expressive if shaped with the potential number of linear regions in mind. 
Hence, we conjecture the architectural choices aimed at maximizing the number of linear regions may lead better performing networks.

That brings us to a gap in the literature: 
to the best of our understanding, 
there is no prior work on how network pruning affects the number of linear regions. 
We take the path that we believe would bring the most insight, which consists of revisiting --- under the lenses of sparsity -- the factors that may limit the maximum number of linear regions based on the neural network architecture. 

In summary, this paper presents the following contributions:
\begin{enumerate}[(i)]
\item We prove an upper bound on the expected number of linear regions over the ways in which weight matrices might be pruned, which refines the bound in \cite{serra2018bounding} to sparsified weight matrices (Section~\ref{sec:bound}).
\item We introduce a network pruning technique based on choosing the density of each layer for increasing the potential number of linear regions (Section~\ref{sec:new}).
\item We propose a method based on Mixed-Integer Linear Programming~(MILP) to count linear regions on input subspaces of arbitrary dimension, 
which generalizes the cases of unidimensional~\cite{hanin2019complexity} and bidimensional~\cite{hanin2019deep} inputs;  
this MILP formulation includes a new constraint in comparison to~\cite{serra2018bounding} for correctly counting linear regions in general  (Section~\ref{sec:counting}). 
\end{enumerate}


\section{Notation}\label{sec:background}

In this paper, we study the linear regions defined by the fully-connected layers of feedforward networks. 
For simplicity, we assume that the entire network consists of such layers and that each neuron has a ReLU activation function, hence being denoted as a \emph{rectifier network}. 
However, our results can be extended to the case in which the fully-connected layers are preceded by convolutional layers, 
and in fact our experiments show their applicability in that context. 
We also abstract the fact that fully-connected layers are often followed by a softmax layer.

We assume that the neural network has an input $\vx = [\evx_1 ~ \evx_2 ~ \dots ~ \evx_{n_0}]^T$ from a bounded domain $\sX$ and corresponding output $\vy = [\evy_1 ~ \evy_2 ~ \dots ~ \evy_m]^T$, and each hidden layer $l \in \sL = \{1,2,\dots,L\}$ has output $\vh^l = [\evh_1^l ~ \evh_2^l \dots \evh_{n_l}^l]^T$ from neurons indexed by $i \in \sN_l = \{1, 2, \ldots, n_l\}$. 
Let $\mW^l$ be the $n_l \times n_{l-1}$ matrix where each row corresponds to the weights of a neuron of layer $l$, $\mW_i^l$ the $i$-th row of $\mW^l$, and $\vb^l$ the vector of biases associated with the units in layer $l$. With $\vh^0$ for $\vx$ and $\vh^{L+1}$ for $\vy$,  the output of each unit $i$ in layer $l$ consists of an affine function $\evg_i^l = \mW_{i}^l \vh^{l-1} + \vb_i^l$ followed by the ReLU activation $\evh_i^l = \max\{0, \evg_i^l\}$. 
We denote the neuron \emph{active} when $\evh_i^l = \evg_i^l > 0$ and \emph{inactive} when $\evh_i^l = 0$ and $\evg_i^l < 0$. 
We explain later in the paper how we consider the special case in which $\evh_i^l = \evg_i^l = 0$.

\section{The Linear Regions of Pruned Neural Networks}\label{sec:bound}

In rectifier networks, small perturbations of a given input produce a linear change on the output before the softmax layer. 
This happens because the neurons that are active and inactive for the original input remain in the same state if the perturbation is sufficiently small. Hence, as long as the neurons remain in their current active or inactive states, 
the neural network acts as a linear function. 

If we consider every configuration of active and inactive neurons that may be triggered by different inputs, then the network acts as a piecewise linear function. 
The theory of linear regions aims to understand what affects the achievable number of such pieces, 
which are also known as linear regions. 
In other words, we are interested in knowing how many different combinations of active and inactive neurons are possible, 
since they 
make the network behave differently for inputs that are sufficiently different from one another. 

Many factors may affect such number of combinations. We consider below some building blocks leading to an upper bound for pruned networks.  

\noindent \textbf{(i) The Activation Hyperplane:} 
Every neuron has an input space corresponding to the output of the neurons from the previous layer, or to the input of the network if the neuron is in the first layer. 
For the $i$-th neuron in layer $l$, that input space corresponds to $\vh^{l-1}$. 
The hyperplane $\mW_{i}^l \vh^{l-1} + \vb_i^l = 0$ defined by the parameters of the neuron separate the inputs in $\vh^{l-1}$ into two half-spaces. 
Namely, the inputs that activate the neuron in one side ($\mW_{i}^l \vh^{l-1} + \vb_i^l > 0$) from those that do not activate the neuron in the other side ($\mW_{i}^l \vh^{l-1} + \vb_i^l < 0$). 
We discuss in (iii) how we regard inputs on the hyperplane ($\mW_{i}^l \vh^{l-1} + \vb_i^l = 0$).

\noindent \textbf{(ii) The Hyperplane Arrangement:} 
With every neuron in layer~$l$ partitioning $\vh^{l-1}$ into two half-spaces, 
our first guess could be that the intersections of these half-spaces would lead the neurons in layer $l$ to partition $\vh^{l-1}$ into a collection of $2^{n_l}$ regions~\cite{montufar2014on}. 
In other words, that there would be one region corresponding to every possible combination of neurons being active or inactive in layer $l$. 
However, the maximum number of regions defined in such a way depends on the number of hyperplanes and the dimension of space containing those hyperplanes. Given the number of activation hyperplanes in layer $l$ as $n_l$ and assuming for now that the size of the input space $\vh^{l-1}$ is $n_{l-1}$, then the number of linear regions defined by layer $l$, or $N_l$, is such that  $N_l \leq \sum_{d=0}^{n_{l-1}} \binom{n_l}{d}$~\cite{zaslavsky1975arrangements}. Since $N_l \ll 2^{n_l}$ when $n_{l-1} \ll n_l$, we note that this bound can be much smaller than initially expected --- and that does not cover the other factors 
discussed in (iv), (v), and (vi). 

\noindent \textbf{(iii) The Boundary:} 
Before moving on, we note that the bound above counts the number of full-dimensional regions defined by a collection of hyperplanes in a given space. 
In other words, the activation hyperplanes define the boundaries of the linear regions and within each linear region the points are such that either $\mW_{i}^l \vh^{l-1} + \vb_i^l > 0$ or $\mW_{i}^l \vh^{l-1} + \vb_i^l < 0$ with respect to each neuron $i$ in layer $l$. Hence, this bound ignores cases in which we would regard $\mW_{i}^l \vh^{l-1} + \vb_i^l = 0$ as making the neuron inactive when $\mW_{i}^l \vh^{l-1} + \vb_i^l \geq 0$ for any possible input in $\vh^{l-1}$, and vice-versa when $\mW_{i}^l \vh^{l-1} + \vb_i^l \leq 0$, 
since in either case the linear region defined with $\mW_{i}^l \vh^{l-1} + \vb_i^l = 0$ would not be full-dimensional and would actually be entirely located on the boundary between other full-dimensional regions. 

\noindent \textbf{(iv) Bounding Across Layers:} 
As we add depth to a neural network, 
every layer of the network breaks each linear region defined so far in even smaller pieces with respect to the input space $\vh^0$ of the network. 
One possible bound would be the product of the bounds for each layer $l$ by assuming the size of the input space to be $n_{l-1}$~\cite{raghu2017expressive}. 
That comes with the assumption that every linear region defined by the first $l-1$ layers can be further partitioned by layer $l$ in as many linear regions as possible. 
However, this partitioning is going to be more detailed in some linear regions than in others because their input space might be very different. 
The output of a linear region in layer $l$ is defined by a linear transformation with rank at most $n_l$.  
The linear transformation would be $\vh^{l} = \mM^l \vh^{l-1} + \vd^l$,  
where $\mM^l_i = \mW^l_i$ and $\vd_i^l = \vb_i^l$ if neuron $i$ of layer $l$ is active in the linear region and $\mM^l_i = \textbf{0}$ and $\vd_i^l = 0$ otherwise. 
Hence, the output from a linear region is the composite of the linear transformations in each layer. 
If layer $l+1$ or any subsequent layer has more than $n_l$ neurons, 
that would not imply that the dimension of the image from any linear region is greater than $n_l$ 
since the output of any linear region after layer $l$ is contained in a space with dimension at most $\min\{n_0, n_1, \ldots, n_{l}\}$~\cite{montufar2017notes}. 
In fact, the dimension the of image is often much smaller if we consider that the rank of each matrix $\mM^l_i$ is bound by how many neurons are active in the linear region, 
and that in only one linear region of a layer we would see all neurons being active~\cite{serra2018bounding}.

\noindent \textbf{(v) The Effect of Parameters:} 
The value of the parameters may also interfere with the hyperplane arrangement. 
First, consider the case in which the rank of the weight matrix is smaller than the number of rows. 
For example, if all activation hyperplanes are parallel to one another and thus the rank of the weight matrix is 1. 
No matter how many dimensions the input space has, this situation is equivalent to drawing parallel lines in a plane. 
Hence, $n_l$ neurons would not be able to partition the input space into more than $n_l+1$ regions. 
In general, it is as if the dimension of the space being partitioned were equal to the rank of the weight matrix~\cite{serra2018bounding}. 
Second,
consider the case in which a neuron is stable, meaning that this neuron is always active or always inactive for any valid input~\cite{tjeng2019stability}. 
Not only that would affect the dimension of the image because a stably inactive neuron always outputs zero, 
but also the effective number of activation hyperplanes:  
since the activation hyperplane associated with a stable neuron has no inputs to one of its sides,  
it does not subdivide any linear region~\cite{serra2020empirical}.

\noindent \textbf{(vi) The Effect of Sparsity:} 
When we start making parameters of the neural network equal to zero through network pruning, we may affect the number of linear regions due to many factors. 
First, some neurons may become stable. 
For example, neuron $i$ in layer $l$ becomes stable if $\mW_i^l = 0$, 
i.e., if that row of parameters only has zeros, 
since the bias term alone ends up defining if the neuron is active ($\vb_i^l > 0$) or inactive ($\vb_i^l < 0$). 
That is also likely to happen if only a few parameters are left, 
such as when all the remaining weights and the bias are all either positive or negative, 
since the probability of all parameters having the same sign increases significantly as the number of parameters left decrease 
if we assume that parameters are equally likely to be positive or negative. 
Second, the rank of the weight matrix $\mW^l$ may decrease with sparsity. 
For example, let us suppose that the weight matrix has $n$ rows, $n$ columns, and that there are only $n$ nonzero parameters. 
Although it is still possible that those $n$ parameters would all be located in distinct rows and columns to result in a full-rank matrix, 
that would only occur in $\dfrac{n!}{\binom{n^2}{n}}$ of the cases if we assume every possible arrangement for those $n$ parameters in the $n^2$ different positions. 
Hence, we should expect some rank deficiency in the weight matrix even if we do not prune that much. 
Third, 
the rank of submatrices on the columns may decrease even if the weight matrix is full row rank. 
This could happen in the typical case where the number of columns exceeds the number of rows, such as when the number of neurons decreases from layer to layer, 
and in that case we could replace the number of active neurons with the rank of the submatrix on their columns for the dimension of the output from each linear region in order to obtain a tighter bound. 

Based on the discussion above, 
we propose an expected upper bound on the number of linear regions over the possible sparsity patterns of the weight matrices. 
We use an expected bound rather than a deterministic one to avoid the unlikely scenarios in which the impact of sparsity is minimal, 
such as in the previous example with $n$ parameters leading to matrix with rank $n$. 
This upper bound considers every possible sparsity pattern in the weight matrix as equally probable, 
which is an assumption that aligns with random pruning and does not seem to be too strict in our opinion. 
For simplicity, we assume that every weight of the network has a probability $p$ of not being pruned; or, conversely, a probability $1-p$ of being pruned. 
We denote $p$ as the network \emph{density}.

Moreover, we focus on the second effect of sparsity --- through a decrease on the rank of the weight matrix --- for two reasons: 
(1) it subsumes part of the first effect when an entire row becomes zero; and 
(2) we found it to be stronger than the third effect in preliminary comparisons with a bound based on it. 

\begin{theorem}\label{thm:main}
Let $R(l,d)$ be the expected maximum number of linear regions that can be defined from layer $l$ to layer $L$ with the dimension of the input to layer $l$ being $d$;  
and let $P(k | R, C, S)$ be the probability that a weight matrix having rank $k$ with $R$ rows, $C$ columns, and probability $S$ of each element being nonzero. 
With $p_l$ as the probability of each parameter in $\mW^l$ from remaining in the network after pruning --- the layer density, then $R(l,d)$ for $l=L$ is at most
\[
\sum\limits_{k=0}^{n_L} P(k | R = n_L, C = n_{L-1}, S=p_L)\sum\limits_{j=0}^{\min\{k,d\}} \binom{n_L}{j}
\]
and $R(l,d)$ for $1 \leq l \leq L-1$ is at most
\[
 \sum\limits_{k=0}^{n_l}P(k | R = n_l, C = n_{l-1}, S=p_l)\sum\limits_{j=0}^{\min\{k,d\}} \binom{n_l}{j} R(l+1, \min\{n_l-j,d,k\}).
\]
\end{theorem}

\begin{proof}
We begin with a recurrence on the number of linear regions similar to the one in~\cite{serra2018bounding}. Namely, let $R(l,d)$ be the maximum number of linear regions that can be defined from layer $l$ to layer $L$ with the dimension of the input to layer $l$ being $d$, 
and let $N_{n_l,d,j}$ be the maximum number of regions from partitioning a space of dimension $d$ with $n_l$ activation hyperplanes such that $j$ of the corresponding neurons are active in the resulting subspaces ($|S^l|=j$):
\begin{align}
    R(l,d) = \left\{
    \begin{array}{cc}
        \sum\limits_{j=0}^{\min\{n_L,d\}} \binom{n_L}{j} & \text{if $l=L$}, \\
         \sum\limits_{j=0}^{n_l} N_{n_l,d,j} R(l+1, min\{j,d\}) & \text{if $1 \leq l \leq L-1$} 
    \end{array}
    \right.
\end{align}

Note that the base case of the recurrence directly uses what we know about the number of linear regions given the number of hyperplanes and the dimension of the space. That bound also applies to $\sum\limits_{j=0}^{n_l} N_{n_l,d,j}$ in the other case from the recurrence. 
Based on Lemma 5 from \cite{serra2018bounding}, $\sum\limits_{j=0}^{n_l} N_{n_l,d,j} \leq \sum\limits_{j=0}^{\min\{n_l, d\}} \binom{n_l}{j}$. 
Some of these linear regions will have more neurons active than others. 
In fact, there are at most $\binom{n_l}{j}$ regions with $|S^l|=j$ for each $j$. 
In resemblance to BC, we can thus assume that the largest possible number of neurons is active in each linear region defined by layer $l$ for the least impact on the input dimension of the following layers. Since $\binom{n_l}{j} = \binom{n_l}{n_l-j}$, we may conservatively assume that $\binom{n_l}{0}$ linear regions have $n_l$ active neurons, $\binom{n_l}{1}$ linear regions have $n_l-1$ active neurons, and so on. 
That implies the following refinement of the recurrence:
\begin{align}\label{eq:bc2}
    R(l,d) = \left\{
    \begin{array}{cc}
        \sum\limits_{j=0}^{\min\{n_L,d\}} \binom{n_L}{j} & \text{if $l=L$}, \\
         \sum\limits_{j=0}^{\min\{n_l,d\}} \binom{n_l}{j} R(l+1, min\{n_l - j,d\}) & \text{if $1 \leq l \leq L-1$} 
    \end{array}
    \right.
\end{align}
Note that there is a slight change on the recurrence call, by which $j$ is replaced with $n_l-j$, given that we are working backwards from the largest possible number of active neurons $n_l$ with $n_l-j$.


Finally, we account for the rank of the weight matrix upon sparsification. 
For the base case of $l=L$, 
we replace $n_L$ from the end of the summation range with the rank $k$ of the weight matrix $\mW^L$, 
and then we calculate the expected maximum number of linear regions using the probabilities of rank $k$ having any value from 0 to $n_L$ as 
\[
\sum\limits_{k=0}^{n_L} P(k | R = n_L, C = n_{L-1})\sum\limits_{j=0}^{\min\{k,d\}} \binom{n_L}{j}, 
\]
which corresponds to the first expression in the statement. 
For the case in which $l \in \{1, \ldots, L-1\}$, 
we similarly replace $n_l$ from the end of the summation range with the rank $k$ of the weight matrix $\mW^l$, 
and then we calculate the expected maximum number of linear regions using the probabilities of rank $k$ having any value from 0 to $n_l$ as 
\[
\sum\limits_{k=0}^{n_l}P(k | R = n_l, C = n_{l-1})\sum\limits_{j=0}^{\min\{k,d\}} \binom{n_l}{j} R^H(l+1, \min\{n_l-j,d,k\}),
\]
which corresponds to the second expression in the statement. \qquad \qquad \qquad \qquad $\blacksquare$
\end{proof}

Please note that the probability of the rank of a sparse matrix is not uniform when the probability of the sparsity patterns is uniform. We discuss how to compute the former from the later as one of the items in Section~\ref{sec:experiments}.

\section{Pruning Based on Linear Regions}\label{sec:new}

Based on Theorem~\ref{thm:main}, 
we devise a network pruning strategy for maximizing the number of linear regions subject to the total number of parameters to be pruned. 
For a global density $p$ reflecting how much should be pruned, 
we may thus choose a density $p_l$ for each layer $l$, 
some of which above and some of which below $p$ if we do not prune uniformly. 
We illustrate below the simpler case of pruning two hidden layers and not pruning the connections to the output layer, 
which is the setting used in our experiments. 
We focused on two layers because there is only one degree of freedom in that case: 
for any density $p_1$ that we choose, the density $p_2$ is implied by $p_1$ and by the global density $p$. 
When there are more layers involved, trying to optimize the upper bound becomes more challenging. 
If the effect is not as strong, it could be due to issues solving this nonlinear optimization problem rather than with the main idea in the paper. 

When pruning two layers, the relevant dimensions for us are the input size $n_0$ and the layer widths $n_1$ and $n_2$. 
Assuming the typical setting in which $n_0 > n_1 = n_2$, 
the maximum rank of both weight matrices is limited by the number of rows ($n_1$ for $\mW^1$ and $n_2$ for $\mW^2$). 
However, the greater number of columns in $\mW^1$ ($n_0$) implies that we should expect the rank of $\mW^1$ to be greater if $p_1 = p_2$, whereas preserving more nonzero elements in $\mW^2$ by pruning a little more from $\mW^ 1$ may change the probabilities for $\mW^2$ with little impact on those for $\mW^1$. 
In some of our experiments, the second layer actually has more parameters than the first, 
meaning that we need to consider $p_1 > p_2$ instead of $p_1 < p_2$.

From preliminary experimentation, 
we indeed observed that (i) pruning more from the layer with more parameters tends to be more advantageous in terms of maximizing the upper bound; 
and also that (ii) the upper bound can be reasonably approximated by a quadratic function. 
Hence, we use the extremes consisting of pruning as much as possible from each of the two layers, say $\overline{p}_1$ and $\overline{p}_2$, in addition to the uniform density $p$ in both layers to interpolate the upper bound. 
If that local maximum of the interpolation is not pruning more from the layer $l$ with more parameters, we search for the density $p_l$ that improves the upper bound the most by uniformly sampling densities from $p$ all the way to $\overline{p}_l$.

\section{Counting Linear Regions in Subspaces}\label{sec:counting}

Based on the characterization of linear regions in terms of which neurons are active and inactive, 
we can count the number of linear regions defined by a trained network with a Mixed-Integer Linear Programming~(MILP) formulation~\cite{serra2018bounding}. 
Among other things, these formulations have also been used for network verification~\cite{cheng2017mip}, embedding the relationship between inputs and outputs of a network into optimization problems~\cite{sanner2017planning,delarue2020rlvrp,bergman2022janos}, 
identifying stable neurons~\cite{tjeng2019stability} to facilitate adversarial robustness verification~\cite{xiao2019training} 
as well as network compression~\cite{serra2020lossless,serra2021compression}, 
and producing counterfactual explanations~\cite{kanamori2021ce}.
Moreover, several studies have analyzed and improved such formulations~\cite{fischetti2018mip,huchette2019ipco,botoeva2020dependency,serra2020empirical,huchette2020mp,serra2021compression}. 

In these formulations, 
the parameters $\mW^l$ and $\vb^l$ of each layer $l \in \sL$ are constant while the decision variables are the inputs of the network ($\vx = \vh^0 \in \sX$), the ouputs before and after activation of each feedforward layer ($\vg^l \in \sR^{n_l}$ and $\vh^l \in \sR_+^{n_l}$ for $l \in \sL$), and the state of the neurons in each layer ($\vz^l \in \{0,1\}^{n_l}$ for $l \in \sL$).  
By mapping these variables according to the parameters of the network, 
we can characterize every possible combination of inputs, outputs, and activation states as distinct solutions of the MILP formulation. 
For each layer $l \in \sL$ and neuron $i \in \sN_l$, 
the following constraints associate the input $\vh^l$ with the outputs $\vg^l_i$ and $\vh^l_i$ as well as with the neuron activation $\vz^l_i$:
\begin{align}
    \mW_i^l \vh^{l-1} + \vb_i^l = \vg_i^l \label{eq:mip_unit_begin} \\
    (\vz_i^l = 1) \rightarrow \vh_i^l = \vg_i^l \label{eq:first_indicator} \\  
    (\vz_i^l = 0) \rightarrow \vg_i^l \leq 0 \\
    (\vz_i^l = 0) \rightarrow \vh_i^l = 0 \label{eq:last_indicator} \\
    \vh_i^l \geq 0 \\
    \vz^l_i \in \{0,1\} \label{eq:mip_unit_end}
\end{align}
The indicator constraints (\ref{eq:first_indicator})--(\ref{eq:last_indicator}) can be converted to linear inequalities~\cite{bonami2015indicator}.

We can use such a formulation for counting the number of linear regions based on the number of distinct solutions on the binary vectors $\vz^l$ for $l \in \sL$. 
However, we must first address the implicit simplifying assumption allowing us to assume that a neuron can be either active ($\vz^l_i = 1$) or inactive ($\vz^l_i=0$) when the preactivation output is zero ($\vg^l_i = 0$) in (\ref{eq:mip_unit_begin})--(\ref{eq:mip_unit_end}). We can do so by maximizing the value of a continuous variable that is bounded by the preactivation output of every active neuron and the negated preactivation output of every inactive neuron. 
In other words, we count the number of solutions on the binary variables for the solutions with positive value for the following formulation:
\begin{align}
    \max ~~~& f \\
    \text{s.t.} ~~~& (\ref{eq:mip_unit_begin})-(\ref{eq:mip_unit_end}) & \forall l \in \sL, i \in \sN_l \\
    & (\vz_i^l = 1) \rightarrow f \leq \vg_i^l & \forall l \in \sL, i \in \sN_l \\
    & (\vz_i^l = 0) \rightarrow f \leq -\vg_i^l & \forall l \in \sL, i \in \sN_l \label{eq:new} \\
    & \vh^0 \in \sX
\end{align}
We note that constraint~(\ref{eq:new}) has not been used in prior work, where it is assumed that the neuron is inactive when $\vg^l_i = 0$~\cite{serra2018bounding,serra2020empirical}. 
However, its absence makes the counting of linear regions incompatible with the theory used to bound the number of linear regions, which assumes that only full-dimensional linear regions are valid. 
Hence, this represents a small correction to count all the linear regions. 

Finally, we extend this formulation for counting linear regions on a subspace of the input. 
This form of counting has been introduced by~\cite{hanin2019complexity} for 1-dimensional inputs  
and later extended by~\cite{hanin2019deep} to 2-dimensional inputs. 
Although far from the upper bound, the number of linear regions can still be very large even for networks of modest size, 
which makes the case for analyzing how neural networks partition subspaces of the input. 
In prior work, 1 and 2-dimensional inputs have been considered as the affine combination of 1 and 2 samples with the origin, 
and a geometric algorithm is used for counting the number of linear regions defined. 
We present an alternative approach by adding the following constraint to the MILP formulation above 
in order to limit the inputs of the neural network:
\begin{align}
    \vh^0 = \vp^0 + \sum_{i=1}^S \alpha_i (\vp^i - \vp^0)
\end{align}
Where $\{ \vp^i \}_{i = 0}^S$ is a set of $S+1$ samples and $\{ \alpha_i \}_{i=1}^S$ is a set of $S$ continuous variables. One of these samples, say $\vp^0$, could be chosen to the be origin. 

\section{Computational Experiments}\label{sec:experiments}

We ran computational experiments aimed at assessing the following items: 
\begin{enumerate}[(1)]
\item \label{exp_item1} if accuracy after pruning and the number of linear regions are connected; 
\item \label{exp_item2} if this connection also translates to the upper bound from Theorem~\ref{thm:main}; and 
\item \label{exp_item3} if that bound can guide us on how much to prune from each layer. 
\end{enumerate}

Our experiments involved models trained on the datasets MNIST~\cite{lecun1998mnist}, 
Fashion~\cite{xiao2017fashion},
CIFAR-10~\cite{krizhevsky2009cifar}, 
and CIFAR-100~\cite{krizhevsky2009cifar}. 
We used multilayer perceptrons having 20, 100, 200, and 400 neurons in each of their 2 fully-connected layers (denoted as $2 \times 20$, $2 \times 100$, $2 \times 200$, and $2 \times 400$), as well as adaptations of the LeNet~\cite{lecun1998mnist} and AlexNet~\cite{krizhevsky2017imagenet} architectures. 
For each choice of dataset and architecture used, we trained and pruned 30 models. 
Only the fully-connected layers were pruned. 
In the case of LeNet and AlexNet, we considered the output of the last convolutional layer as the input for upper bound calculations, 
as if their respective dimensions were $400 \times 128 \times 84$ and $1024 \times 4096 \times 4096$.
We removed the weights with smallest absolute value (magnitude pruning), 
using either the same density $p$  on each layer or choosing different densities while pruning the same number of parameters in total. 
We defer other details about the experiments for after presenting what we tested and our findings.

\noindent \textbf{Experiment 1:} 
We compared the mean accuracy of networks that are pruned uniformly according to their network density with the number of linear regions on subspaces defined by random samples from the datasets (Figure~\ref{fig:exact_acc}) as well as with the upper bound with input dimensions matching those subspaces (Figure~\ref{fig:exact_ub}). 
We used a simpler architecture ($2 \times 20$) to keep the number of linear regions small enough to count
and a simpler dataset (MNIST) to obtain models with good accuracy. 
In this experiment, 
we observe that indeed the number of linear regions drops with network density and consequently with accuracy. 
However, the most relevant finding is that the upper bound also drops in a similar way, 
even if its values are much larger. 
This finding is important because it is actionable: 
if we compare the upper bound resulting from different pruning strategies, 
then we may prefer a pruning strategy that leads to a smaller drop in the upper bound. 
Moreover, it is considerably cheaper to work with the upper bound 
since we do not need to train neural networks and neither count their linear regions. 

\begin{figure}[h!]
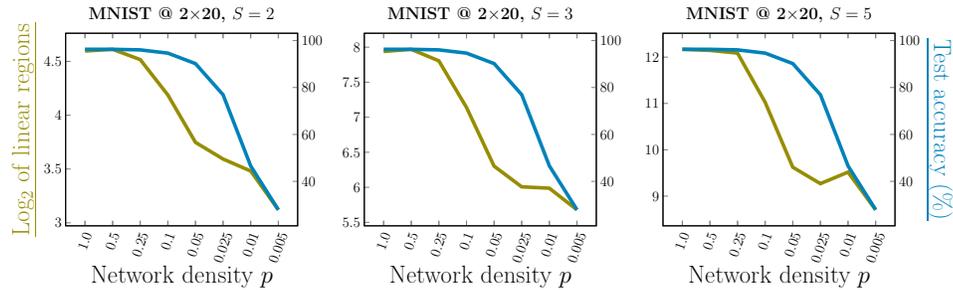

    \centering
\acclrPlot{0.45}{Perceptron20_exact_2_MNIST}{MNIST @ 2$\times$20, $S=2$}{Log$_2$ of linear regions}{}{-0.5cm}
\acclrPlot{0.45}{Perceptron20_exact_3_MNIST}{MNIST @ 2$\times$20, $S=3$}{}{}{-0.5cm}
\acclrPlot{0.45}{Perceptron20_exact_5_MNIST}{MNIST @ 2$\times$20, $S=5$}{}{Test accuracy (\%)}{-0.5cm}
\caption{Comparison between 
 mean number of linear regions on the affine subspace defined by $S=2$, 3, or 5 sample points (\textcolor{olive}{\uline{olive}} curve) 
and mean test accuracy (\textcolor{bmblue}{\uline{blue}} curve; right y axis) 
with the same density $p$ used to prune both layers of the networks.}
\label{fig:exact_acc}
\end{figure}

\begin{figure}
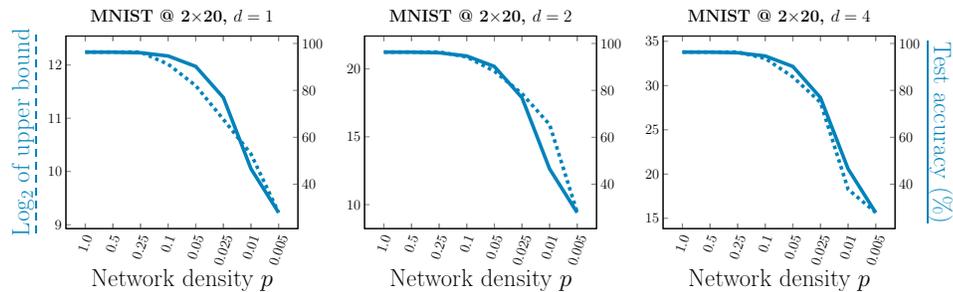

    \centering
\uboundPlot{0.45}{Perceptron20_exact_2_MNIST}{MNIST @ 2$\times$20, $d=1$}{Log$_2$ of upper bound}{}{-0.5cm}
\uboundPlot{0.45}{Perceptron20_exact_3_MNIST}{MNIST @ 2$\times$20, $d=2$}{}{}{-0.5cm}
\uboundPlot{0.45}{Perceptron20_exact_5_MNIST}{MNIST @ 2$\times$20, $d=4$}{}{{Test accuracy (\%)}}{-0.5cm}
\caption{Comparison between 
the upper bound from Theorem~\ref{thm:main} (\textcolor{bmblue}{\dashuline{dashed blue}} curve) 
for input dimension $d=1$, 2, and 4 (equivalent to $S=2$, 3, and 5) 
and mean test accuracy (\textcolor{bmblue}{\uline{continuous blue}} curve; right y axis) 
for the same networks and densities from Figure~\ref{fig:exact_acc}.}
    \label{fig:exact_ub}
\end{figure}

\noindent \textbf{Experiment 2:} 
We compared using the same density $p$ in each layer 
with using per layer densities as described in Section~\ref{sec:new}. 
We evaluated the simpler datasets (MNIST, Fashion, and CIFAR-10) on the simpler architectures (multilayer perceptrons and LeNet) in Figure~\ref{fig:accuracy}, where every combination of dataset and architecture is tested to compare gains across network sizes and datasets. 
We set aside the most complex architecture (AlexNet) and the most complex datasets (CIFAR-10 and CIFAR-100) in Figure~\ref{fig:accuracy_alexnet}. 
In this experiment, 
we observe that pruning the fully-connected layers differently and oriented by the upper bound indeed leads to more accurate networks. 
The difference between the pruning strategies is noticeable once the network density starts impacting the network accuracy. 
We intentionally evaluated network densities leading to very different accuracies and all the way to a complete deterioration of network performance, and we notice that the gain is consistent across all of them them. 
If the number of parameters is similar across fully-connected layers, such as in the case of $2 \times 400$, 
we notice that the gain is smaller because more uniform densities are better for the upper bound. 
Curiously, we also observe a relatively greater gain with our pruning strategy for CIFAR-10 in the case of multilayer perceptrons.

\begin{figure}
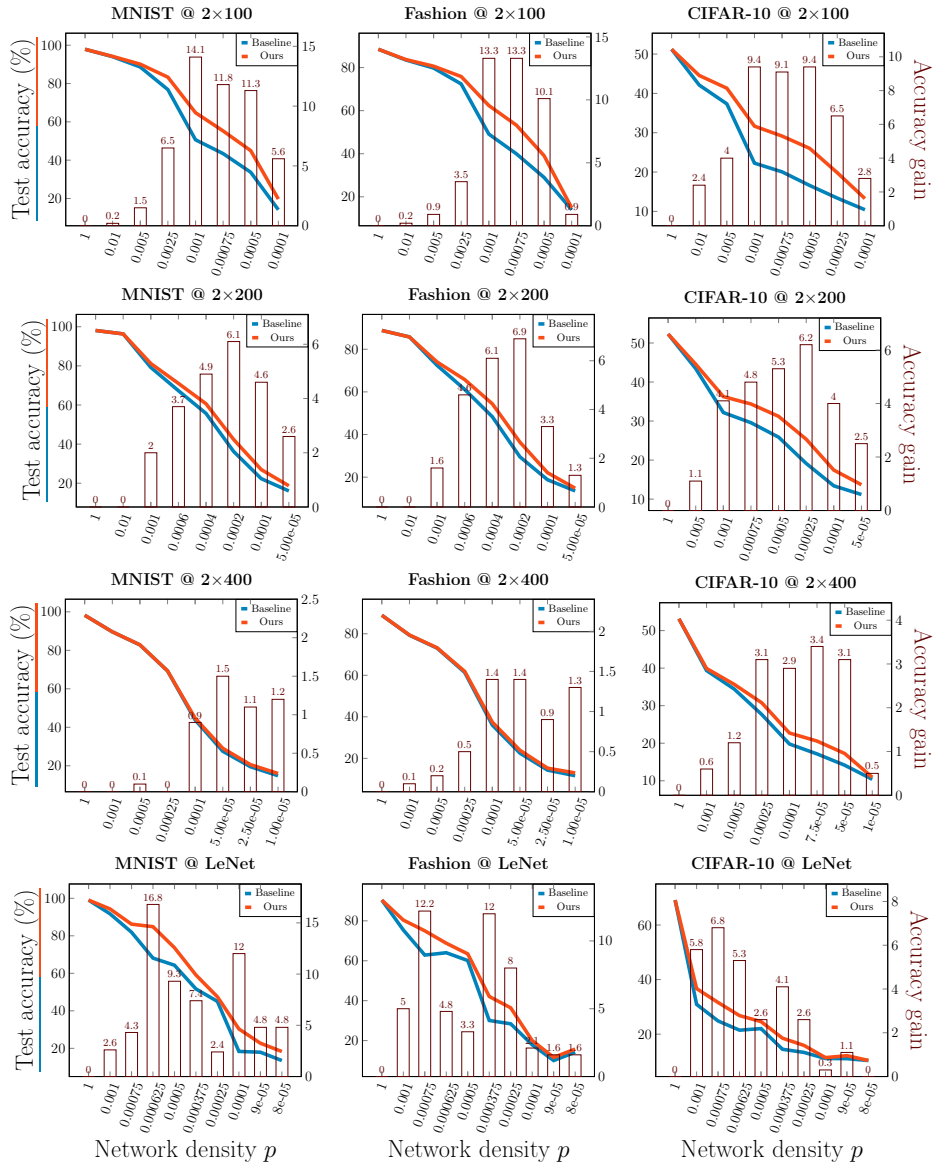

    \centering
\accuracyPlot{0.45}{mnist100}{MNIST @ 2$\times$100}{upper,abs value=2}{\underlineAccuracyLabel}{}{}{}{-0.5cm}
\accuracyPlot{0.45}{fashion100}{Fashion @ 2$\times$100}{upper,abs value=2}{}{}{}{}{-0.5cm}
\accuracyPlot{0.45}{cifar100}{CIFAR-10 @ 2$\times$100}{upper,abs value=2}{}{Accuracy gain}{}{}{-0.5cm}

\accuracyPlot{0.45}{mnist200}{MNIST @ 2$\times$200}{upper,abs value=1}{\underlineAccuracyLabel}{}{}{}{-0.5cm}
\accuracyPlot{0.45}{fashion200}{Fashion @ 2$\times$200}{upper,abs value=1}{}{}{}{}{-0.5cm}
\accuracyPlot{0.45}{cifar200}{CIFAR-10 @ 2$\times$200}{upper,abs value=1}{}{Accuracy gain}{}{}{-0.5cm}

\accuracyPlot{0.45}{mnist400}{MNIST @ 2$\times$400}{upper,abs value=1}{\underlineAccuracyLabel}{}{}{}{-0.5cm}
\accuracyPlot{0.45}{fashion400}{Fashion @ 2$\times$400}{upper,abs value=1}{}{}{}{}{-0.5cm}
\accuracyPlot{0.45}{cifar400}{CIFAR-10 @ 2$\times$400}{upper,abs value=1}{}{Accuracy gain}{}{}{-0.5cm}

\accuracyPlot{0.45}{mnistlenet}{MNIST @ LeNet}{upper,abs value=2}{\underlineAccuracyLabel}{}{Network density $p$}{-0.37}{-0.5cm}
\accuracyPlot{0.45}{fashionlenet}{Fashion @ LeNet}{upper,abs value=2}{}{}{Network density $p$}{-0.37}{-0.5cm}
\accuracyPlot{0.45}{cifarlenet}{CIFAR-10 @ LeNet}{upper,abs value=2}{}{Accuracy gain}{Network density $p$}{-0.37}{-0.5cm}
    
\caption{Comparison between 
the mean {test} accuracy as fully-connected layers are pruned {using the baseline method and our method} with each network density $p$. {In the \textbf{baseline} method, the same density is used in all layers (\textcolor{bmblue}{\uline{blue}} curve). In \textbf{our method}, layer densities are chosen to maximize the bound from Theorem~\ref{thm:main} while pruning the same number of parameters (\textcolor{borange}{\uline{orange}} curve)}. The \textbf{accuracy gain} from using our method instead of the baseline is shown in the scaled columns (\textcolor{cred}{\textbf{maroon}} bars; right y axis). Each column refers to a dataset among MNIST, Fashion, and CIFAR-10. 
Each row refers to an architecture among multilayer perceptrons ($2 \times 100$, $2 \times 200$, and $2 \times 400$) and LeNet. 
We test every combination of dataset and architecture. 
}
    \label{fig:accuracy}
\end{figure}

\begin{figure}
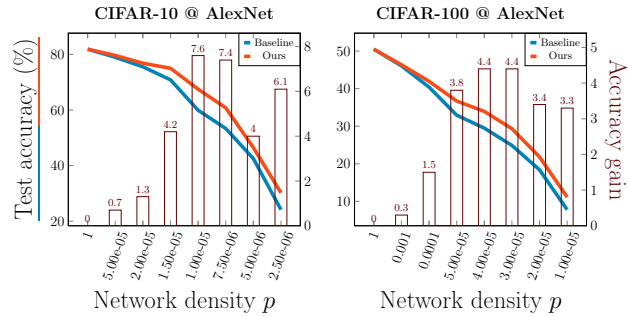

    \centering
\accuracyPlot{0.45}{alexnet_cifar10}{CIFAR-10 @ AlexNet}{upper, abs value= 1}{\underlineAccuracyLabel}{}{Network density $p$}{-0.38}{-0.5cm}
\accuracyPlot{0.45}{alexnet_cifar100}{CIFAR-100 @ AlexNet}{upper, abs value= 1}{}{Accuracy gain}{Network density $p$}{-0.38}{-0.5cm}
    \caption{Comparison between mean test accuracy for the same strategies as in Figure~\ref{fig:accuracy} for the AlexNet architecture, in which we test the datasets CIFAR-10 and CIFAR-100.}
    \label{fig:accuracy_alexnet}
\end{figure}


\noindent \textbf{Additional Details:} 
Each network was trained for 15 epochs using stochastic gradient descent with batch size of 128 and learning rate of 0.01, pruned, and then fine-tuned with the same hyperparameters for another 15 epochs. 
We have opted for magnitude-based pruning due to its simplicity, popularity, and frequent use as a 
component of more sophisticated pruning algorithms~\cite{blalock2020survey,frankle2019lottery}. 
Our implementation is derived from the ShrinkBench framework~\cite{blalock2020survey}. 
In the baseline that we used, 
we opted for removing a fixed proportion of parameters from each layer (layerwise pruning) 
to avoid disconnecting the network, which we observed to happen under extreme sparsities if the parameters with smallest absolute value were mostly concentrated in one of the layers. 
We measured the mean network accuracy before pruning, which corresponds to network density $p=1$, 
as well as for another seven values of $p$. 
In the experiments in Figure~\ref{fig:accuracy}, 
the choices of $p$ were aimed at gradually degrading the accuracy toward random guessing, which corresponds to accuracy 10\% accuracy in those datasets with 10 balanced classes (MNIST, Fashion, and CIFAR-10). 
In the experiments with AlexNet in Figure~\ref{fig:accuracy_alexnet}, we aimed for a similar decay in performance. 

\noindent \textbf{Upper Bound Calculation:} 
Estimating the probabilities $P(k|R,C,S)$ in Theorem~\ref{thm:main} is critical to calculate the upper bound. 
For multilayer perceptrons and LeNet, 
we generated a sample of matrices with the same shape as the weight matrix for each layer and in which every element is randomly drawn from the normal distribution with mean 0 and standard deviation 1. 
These matrices were randomly pruned based on the density $p$, which may have been the same for every layer or may varied per layer as discussed later, 
and then their rank was calculated. 
We first generated 50 such matrices for each layer, kept track the minimum and maximum rank values obtained, $\min_r$ and $\max_r$, and then generated more matrices until the number of matrices generated was at least as large as $(\max_r-\min_r+1)*50$. 
For example, 50 matrices are generated if the rank is always the same, 
and 500 matrices are generated if the rank goes from 11 to 20. 
Finally, 
we calculated the probability of each possible rank based on how many times that value was observed in the samples. For example, if 10 out of 500 matrices have rank 11, then we assumed a probability of 2\% for the rank of the matrix to be 11. 
For AlexNet, 
the time required for sampling is considerably longer. 
Hence, we resorted to an analytical approximation which is faster but possibly not as accurate. 
For an $m \times n$ matrix, $m \leq n$, with density $p$, 
the probability of all the elements being zero in a given row is $(1-p)^n$. 
We can overestimate the rank of the matrix as the number of rows with nonzero elements, 
which then corresponds to a binominal probability distribution with $m$ independent trials having each a probability of success given by $1 - (1-p)^n$. 
For $2 \times 100$, calculating the upper bound takes 15-20 seconds with sampling and 0.5-1 second with the analytical approximation. For $2 \times 400$, we have 10-20 minutes vs. 20 seconds. For AlexNet, the analytical approximation takes 20 minutes.





\section{Conclusion}\label{sec:conclusion}

In this work, 
we studied how the theory of linear regions can help us identify how much to prune from each fully-connected feedforward layer of a neural network. 
First, 
we proposed an upper bound on the number of linear regions based on the density of the weight matrices when neural networks are pruned. 
We observe from Figure~\ref{fig:exact_ub} that the upper bound is reasonably aligned with the impact of pruning on network accuracy. 
Second, 
we proposed a method for counting the number of linear regions on subspaces of arbitrary dimension. 
In prior work, 
the counting of linear regions in subspaces is restricted to at most 3 samples and thus dimension 2~\cite{hanin2019deep}. 
We observe from Figure~\ref{fig:exact_acc} to the number of linear regions is also aligned with the impact of pruning on network accuracy --- although not as accurately as the upper bound. 
Third, and most importantly, 
we leverage this connection between the upper bound and network accuracy under pruning to decide how much to prune from each layer subject to an overall network density $p$. 
We observe from Figure~\ref{fig:accuracy} 
that we obtain considerable gains in accuracy across varied datasets and architectures by pruning from each layer in a proportion that improves the upper bound on the number of linear regions rather than pruning uniformly. 
These gains are particularly more pronounced when the number of parameters differs across layers.  
Hence, 
the gains are understandably smaller when the width of the layers increases (from 100 to 200 and 400) but greater when the size of the input increases (from 784 for MNIST and Fashion to 3,072 for CIFAR-10 with a width of 400). 
We also obtain positive results with pruning fully connected layers of convolutional networks as illustrated with LeNet and AlexNet, 
and in future work we intend to investigate how to also make decisions about pruning convolutional filters. 
Althought we should not discard the possibility of a confounding factor affecting both accuracy and linear regions, 
our experiments indicate that the potential number of linear regions can guide us on pruning more from neural networks with less impact on accuracy.

\paragraph{Acknowledgement}

We would like to thank the anonymous reviewers for their constructive feedback to improve this paper. 
Junyang Cai, Khai-Nguyen Nguyen, Nishant Shrestha, Aidan Good, Ruisen Tu, and Thiago Serra were supported by the National Science Foundation (NSF) grant IIS 2104583. 
Xin Yu and Shandian Zhe were supported by the NSF CAREER Award IIS 2046295.

%
%
%
\bibliographystyle{splncs04}
\bibliography{references}

\end{document}

%% file: math_commands.tex
\usepackage{amsmath,amsfonts,bm}

\def\eqref#1{equation~\ref{#1}}

\def\1{\bm{1}}

\def\vb{{\bm{b}}}

\def\vd{{\bm{d}}}

\def\vg{{\bm{g}}}
\def\vh{{\bm{h}}}

\def\vp{{\bm{p}}}

\def\vx{{\bm{x}}}
\def\vy{{\bm{y}}}
\def\vz{{\bm{z}}}

\def\evg{{g}}
\def\evh{{h}}

\def\evx{{x}}
\def\evy{{y}}

\def\mM{{\bm{M}}}

\def\mW{{\bm{W}}}

\DeclareMathAlphabet{\mathsfit}{\encodingdefault}{\sfdefault}{m}{sl}
\SetMathAlphabet{\mathsfit}{bold}{\encodingdefault}{\sfdefault}{bx}{n}

\def\sL{{\mathbb{L}}}

\def\sN{{\mathbb{N}}}

\def\sR{{\mathbb{R}}}

\def\sX{{\mathbb{X}}}

